\documentclass{article}

    \usepackage[preprint]{neurips_2025}

\usepackage[utf8]{inputenc} 
\usepackage[T1]{fontenc}    
\usepackage{hyperref}       
\usepackage{url}            
\usepackage{booktabs}       
\usepackage{amsfonts}       
\usepackage{nicefrac}       
\usepackage{microtype}      
\usepackage{xcolor}         

\usepackage[pdftex]{graphicx}
\usepackage{subfigure}
\usepackage{subcaption}

\usepackage{amsmath}
\usepackage{amssymb}
\usepackage{mathtools}
\usepackage{amsthm}

\usepackage[capitalize,noabbrev]{cleveref}

\theoremstyle{plain}
\newtheorem{theorem}{Theorem}[section]
\newtheorem{proposition}[theorem]{Proposition}
\newtheorem{lemma}[theorem]{Lemma}
\newtheorem{corollary}[theorem]{Corollary}
\theoremstyle{definition}
\newtheorem{definition}[theorem]{Definition}

\theoremstyle{remark}

\usepackage{amsmath,amssymb,amsthm,bbm,mathtools}
\usepackage{algorithm,algorithmicx,algpseudocode}
\usepackage{graphicx,subfigure}
\usepackage{threeparttable}
\usepackage{enumitem}
\usepackage{multirow}
\usepackage{url}
\usepackage{bm}
\usepackage{color}
\usepackage[all]{xy}
\usepackage{hyperref}
\usepackage{xspace}
\usepackage{booktabs}
\usepackage{soul}
\usepackage{lipsum}

\graphicspath{{figs/}}

\newcommand{\model}{RWPT\xspace}

\newcommand{\ve}{\mathbf{e}}

\newcommand{\vh}{\mathbf{h}}

\newcommand{\vv}{\mathbf{v}}

\newcommand{\vx}{\mathbf{x}}

\newcommand{\mE}{\mathbf{E}}

\newcommand{\mH}{\mathbf{H}}

\newcommand{\tB}{\mathcal{B}}

\newcommand{\tL}{\mathcal{L}}

\newcommand{\Romannumber}[1]{\uppercase\expandafter{\romannumeral #1}}

\newcommand{\expect}{\mathbb{E}}

\newcommand{\zeros}{\mathbf{0}}

\title{Toward a Graph Foundation Model: \\ Pre-Training Transformers With Random Walks}

\author{
  Ziyuan Tang\thanks{Most of this work was completed during the author's internship at MIT-IBM Watson AI Lab.} \\
  Department of Computer Science \\ \& Engineering \\
  University of Minnesota\\
  \texttt{tang0389@umn.edu}
  \And
  Jie Chen \\
  MIT-IBM Watson AI Lab \\
  IBM Research \\
  \texttt{chenjie@us.ibm.com}
}

\begin{document}

\maketitle

\begin{abstract}
  A foundation model like GPT elicits many emergent abilities, owing to the pre-training with broad inclusion of data and the use of the powerful Transformer architecture. While foundation models in natural languages are prevalent, can we build similar models for graphs? This paper describes an approach toward a graph foundation model that is pre-trained with diverse graph datasets by adapting the Transformer backbone. A central challenge toward this end is how a sequence model encodes graphs of varying sizes and from different domains. We propose representing a node as multiple random walks, such that the Transformer can extract node representations from sequences, which in turn form edge and graph representations. We develop a novel context prediction loss for these random walks and theoretically analyze their expressive power in distinguishing neighborhoods and graphs. We also demonstrate the pre-training of our model and its adaptation to downstream tasks, showcasing its potential as a foundation for processing and reasoning with graph-structured data.
\end{abstract}

\section{Introduction}\label{sec:intro}

The concept of a \emph{foundation model} was conceived based on substantial evidence suggesting that a pre-trained Transformer can solve many natural language tasks when properly adapted~\citep{Bommasani2021,Vaswani2017}. Moreover, when scaled with sufficient parameters, these Transformers elicit
emergent abilities that are not present in smaller models~\citep{Wei2022}. A key to sustaining the performance of these Transformers is the concurrent scaling of the training data and the model size~\citep{Kaplan2020,Hoffmann2022}. Informally speaking, the current large language models (LLMs) are trained on at least the entire Internet.

We ask if a similar foundation model exists for graphs~\citep{Mao2024}. Such a model is pre-trained with and adapted to various kinds of graphs. Unlike natural language data that are sequential, graph data pose unique challenges for Transformers. First, graph data are non-sequential. Invariance and equivariance to node permutations require forgoing the positional encoding that is crucial for sequence Transformers; the graph structure will need to be encoded differently (such as using structural encoding or attention bias). Second, graph sizes vary, with node counts ranging from fewer than ten to several billion in practice. Meanwhile, the context length of a typical LLM is on the order of thousands, causing difficulties in batching graphs. While the nodes of small graphs can be chained together to fill a sequence, a large graph will need to be partitioned such that its nodes form multiple sequences. In this case, the connection between different sequences is lost.

\emph{Our objective is to develop a methodology toward building a foundation model for graphs}, with the following desiderata:
\begin{itemize}[leftmargin=7mm]
\item[\textbf{D1:}] The model should be pre-trained with a broad inclusion of graph datasets in a self-supervised manner without the influence of task labels.
\item[\textbf{D2:}] The model can be adapted to any downstream tasks and transferred to graphs in a new domain.
\item[\textbf{D3:}] The model can handle graphs of varying sizes, ranging from small molecules to large networks.
\item[\textbf{D4:}] The model can capture long-range interactions when they are important for the task at hand.
\end{itemize}

These desiderata require a holistic design of the model architecture, input and output formats, training objectives, and downstream adaptation strategies. We begin with interrogating the strengths and limitations of the two most widely used graph deep learning architectures: Graph Neural Networks (GNNs)~\citep{Li2016,Kipf2017,Hamilton2017,Gilmer2017,Velickovic2018,Xu2019,Gasteiger2019,Chen2020} and Graph Transformers (GTs)~\citep{Dwivedi2021,Kreuzer2021,Ying2021,Wu2021,Rampasek2022,Chen2022,Ma2023}. The computation pattern of GNNs is neighborhood aggregation; as a result, the main challenge is a uniform network depth for all graphs and the handling of long-range interactions if exist~\citep{Dwivedi2022}. Experience suggests that GNNs for different domains vary substantially in depth. While one may attempt to take the maximum depth, which also resolves the long-range challenge, on other occasions deep GNNs suffer from the over-smoothing problem~\citep{Li2018}. Mitigation approaches exist, such as residual connections~\citep{Chen2020} and edge removals~\citep{Rong2020}, but many layers aggravate the neighborhood explosion problem because typical neighborhood sampling methods~\citep{Hamilton2017,Chen2018} will still create an enormous neighborhood. In the pursuit of a foundation model, recent approaches~\citep{Liu2024,Wang2024} sample small-hop neighborhoods for large graphs so that the GNN depth can be more flexible, but these neighborhoods still miss long-range information. On the other hand, GTs are a principled approach to incorporating this information because of the pairwise attention, but they are faced with a different challenge---scalability. For a graph with $n$ nodes, it typically takes $O(n^2)$ time to compute the attention scores. Much effort has been devoted to scaling GTs to large $n$, such as (i) using kernel approximation of the softmax attention~\citep{Wu2022,Choromanski2021}, (ii) taking a hierarchical approach~\citep{Zhu2023}, and (iii) changing the input from a sequence of all graph nodes to a sequence of sampled neighbors of one node~\citep{Zhao2021,Zhang2022,Chen2023}. However, approaches (i) and (ii) still have trouble with batching when graphs have varying sizes and approach (iii) weakens the incorporation of long-range information.

In this work, we propose the \underline{R}andom \underline{W}alk-Based \underline{P}re-Trained \underline{T}ransformer (\model). The main idea behind this model is the use of multiple random walks to represent one node and the retention of the Transformer backbone for its foundational nature in representation learning. RWPT differs from usual GTs in that the Transformer input is neither the whole graph nor a sequence of sampled neighbors. Instead, multiple random walks are taken from a root node, forming ordered sequences including near neighbors and faraway nodes. Random walks are a revival of the early node embedding methods prior to GNNs, such as DeepWalk~\citep{Perozzi2014} and node2vec~\citep{Grover2016}, which permit favoring depth in addition to breadth when considering node co-occurrences. They are key to our holistic design that meets the four aforementioned desiderata: Random walks resolve the batching problem of GTs when training graphs have drastically different sizes (\textbf{D3}); they encode a larger receptive field and better cope with long-range interactions~\citep{Chen2024}, compared with small-hop neighborhood sampling (\textbf{D4}); and they allow the pre-training with any cumulation of graph datasets for scaling (\textbf{D1}) as well as the separation of self-supervised pre-training and downstream adaption (\textbf{D2}), following closely the practice of LLMs for natural language data. Moreover, we theoretically show that random walks with shortest-path distance positional encoding can reconstruct any ball 
(the ego-graph of a node induced by its $r$-hop neighborhood)
and distinguish two balls up to isomorphism. Hence, they are expressive in node representation learning.

Our contributions are as follows:
\begin{itemize}[leftmargin=*]
\item We position four desiderata for a graph foundation model. The first two are parallel to natural language models while the other two are unique to graph-structured data.
\item We propose \model that meets these requirements and addresses the limitations of current graph deep learning models (GNNs and GTs). Central to \model is the use of multiple random walks to represent a node, which subsequently invokes the accompanying designs of positional encoding, attention masking, and training loss for the Transformer.
\item We conduct a theoretical analysis on random walks and show their expressivity in distinguishing node neighborhoods, justifying their use for representation learning.
\item We conduct comprehensive experiments to demonstrate the effectiveness of \model compared with (semi-)supervised and self-supervised methods, highlighting its transferability and adaptivity in cross-domain and cross-task uses.
\end{itemize}

\section{Related work}\label{sec:rw}
This work emerges during the rapid development of foundation models for natural language processing (synonymously called LLMs). Efforts for building a unified model for graph data concurrently occur. Our approach is the most relevant to two recent methods: OFA~\citep{Liu2024} and GFT~\citep{Wang2024}.

OFA~\citep{Liu2024} trains a single GNN across multiple datasets and tasks in a supervised fashion. For node- and link-level tasks, the model operates on $\ell$-hop neighborhood subgraphs rather than the entire graph. It remains unclear whether this framework can be adapted for pretraining without task-specific supervision.
GFT~\citep{Wang2024} pretrains a vocabulary of embedding vectors derived from quantized $\ell$-hop neighborhoods, which are encoded using a GNN. During inference, it assigns the nearest pretrained embedding to a new neighborhood subgraph for downstream tasks. Its reliance on small $\ell$-hop neighborhoods limits its ability to capture long-range dependencies. To address this, we utilize random walks that can reach faraway contexts.

See Section~\ref{appendix:rw} for a more comprehensive discussion of the related work, including GTs, pre-training GNNs, foundation model concepts in GNNs, and LLM-based methods to solve graph problems.

\section{Methodology}\label{sec:methodology}
In this section, we elaborate on the details of the proposed \model model. It is a holistic design that involves not only the neural architecture but also the data representation and training. Three aspects are highlighted: the formulation and encoding of the input sequence, the attention mechanism, and the pre-training loss. 
Figure~\ref{fig:method} shows feedforward flow of \model during inference.

\begin{figure}[t]
  \centering
  \includegraphics[width=\linewidth]{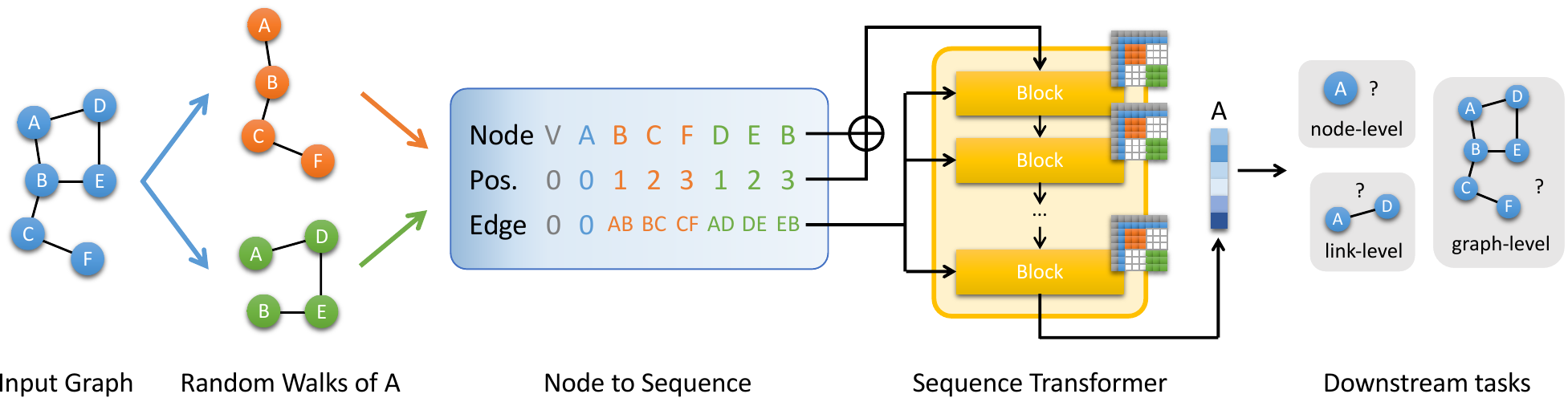}
  \vskip 1pt
  \caption{Pipeline of \model. Each node is represented by multiple random walks formulated into one positionally encoded sequence, augmented with domain information. The sequence is processed by a Transformer with a per-walk attention mask. The node representation is extracted from the output. The model is fine-tuned through training a prediction head dedicated to the downstream task.}
  \label{fig:method}
\end{figure}

\subsection{Random-walk representation of nodes}
\label{sec:rw}

Random walks form ordered sequences, which are natural inputs to the Transformer. In action, assume that $i=i_0$ is the node of interest (i.e., the root node). We run $k$ independent random walks of length $\ell$ starting from $i_0$ and denote by $i_s^r$ the node at step $s=1,\ldots,\ell$ and walk index $r=1,\ldots,k$. We concatenate all nodes, walk by walk, to form a sequence
\begin{equation}\label{eq:catseq}
  \text{seq}(i) = [i_0, i_1^1,\ldots,i_{\ell}^1, i_1^2,\ldots,i_{\ell}^2, \ldots, i_1^k,\ldots,i_{\ell}^k ].
\end{equation}
One may consider that the union of the random walks forms a sampling of the large-hop neighborhood of $i$. This sampled neighborhood differs from a common one resulting from neighborhood sampling~\citep{Hamilton2017,Chen2018} in that a larger walk length $\ell$ is permissible but the number of hops in neighborhood sampling is limited, because the number of sampled nodes is exponential in the hop count. Note that a node in $\text{seq}(i)$ may appear multiple times (because it is sampled by different walks), but its embedding in different appearances may be different because of edge features.

For a foundation model, we intend to pre-train it with multiple datasets. Then, to distinguish graphs from different domains, we introduce a virtual token $v$ for each dataset. We prepend this token to the sequence~\eqref{eq:catseq}; that is, the full sequence for a node $i$ input to \model is
\begin{equation}\label{eq:seq_input}
  s(i) = [v, \text{seq}(i)].
\end{equation}

\subsection{Sequence encoding}\label{sec:seq_enc}
Like a standard Transformer whose input sequence is encoded by token embedding and positional encoding, we define our input node features and positional encoding for \model. Additionally, we incorporate edge features by formulating them into the sequence and adding them to the input of each Transformer block.

\textbf{Unified input node features.}
One of the technical difficulties in designing a graph foundation model is unifying graphs from different domains with varying node features and dimensions. LLM offers a perfect mechanism to mitigate this difficulty~\citep{Chen2023a,He2023,Liu2024}. Nearly all graphs from practical problems are equipped with semantic meanings for their nodes, edges, and even themselves as a whole. For example, the nodes in a molecular graph are atoms and the nodes in a citation graph are papers. They all can be described by text. Hence, we use an LLM to process the textual information, $t_i$, of a node $i$, yielding the node feature vector
\begin{equation}
  \vx_i = \text{LLM}(t_i).
\end{equation}
An advantage of obtaining node features in this manner is that the LLM, as a text foundation model, unifies knowledge of different domains and offers the same output dimension for nodes of any domain.
Even for non-textual graphs, we can leverage an LLM to summarize structural features and generate descriptions~\citep{Liu2024,Wang2024}.
For example, a node comes with local degree profiles and centrality measures, from which the textural description can be ``a node with medium degree and high betweenness centrality; its value on the Fiedler vector belongs to the 90 percentile.''

Similarly, for an edge $ij$ of a graph and the virtual token $v$ of a graph dataset, let their textual description be $t_{ij}$ and $t_v$. Then, we obtain the edge feature and virtual-node feature
\begin{equation}
  \ve_{ij} = \text{LLM}(t_{ij}), \qquad \vv = \text{LLM}(t_v),
\end{equation}
respectively. The virtual-node feature $\vv$ will be used together with node features in the input sequence; the use of the edge feature $\ve_{ij}$ will be elaborated later.

\textbf{Positional encoding.}
We enhance the integration of the graph structure by leveraging positional encodings based on shortest-path (SP) distances. Specifically, for a node $i_s^r$ in the rooted sequence $\text{seq}(i_0)$, its position is defined as the SP distance from the root $i_0$ to $i_s^r$, which is at most $s$.
Additionally, for the virtual token $v$ and the root token $i_0$, their position is $0$. 

The positional encoding is used in the subsequent theoretical analysis of random walks. We can straightforwardly extract SPs from the walks: if the positions of nodes on a walk segment from $u$ to $v$ are monotonically increasing by 1, then this segment must be a shortest path between $u$ and $v$.

\textbf{Incorporating edge features.}
Edge features can be used to enhance the encoding of a node sequence. For the $r$th walk $i_0,i_1^r,\ldots,i_{\ell}^r$, we form an edge-feature sequence
\begin{equation}
  \mE^r = [\ve_{i_0,i_1^r}, \ve_{i_1^r,i_2^r}, \ldots, \ve_{i_{\ell-1}^r,i_{\ell}^r}],
\end{equation}
and we concatenate the $k$ walks and prepend two zero vectors to form the full sequence
\begin{equation}
  \mE = [\zeros, \zeros, \mE^1, \mE^2, \ldots, \mE^k],
\end{equation}
which has the same length as the Transformer input.

Rather than merely adding $\mE$ to the Transformer input, we project $\mE$ and add it to the input of each Transformer block, similar to how the edge information is processed in every layer of a GNN~\citep{Hu2020a}. Specifically, let the $t$th block of a standard Transformer~\citep{Vaswani2017,Liu2018} be
$\mH^{(t+1)} = \text{Block}(\mH^{(t)})$.
We introduce a block-dependent projector (a linear layer), Proj, and modify the block to be
$\mH^{(t+1)} = \text{Block}(\mH^{(t)} + \text{Proj}^{(t)}(\mE))$.
The projectors are mainly used to map data from the LLM output dimension to the Transformer embedding dimension, similar to the one for node features.

\subsection{Per-walk attention mask}\label{sec:attn_mask}
The query-key-value (QKV) attention mechanism typically comes with a mask on the QK product (attention) matrix before softmax. The effect of this mask is to ignore some V items when linearly combining them. For example, the upper triangular part of the attention matrix is masked out in a causal Transformer, because a token will depend on the past but not the future information.

\begin{figure}[t]
\begin{minipage}{0.48\linewidth}
    \centering
    \includegraphics[width=.82\linewidth]{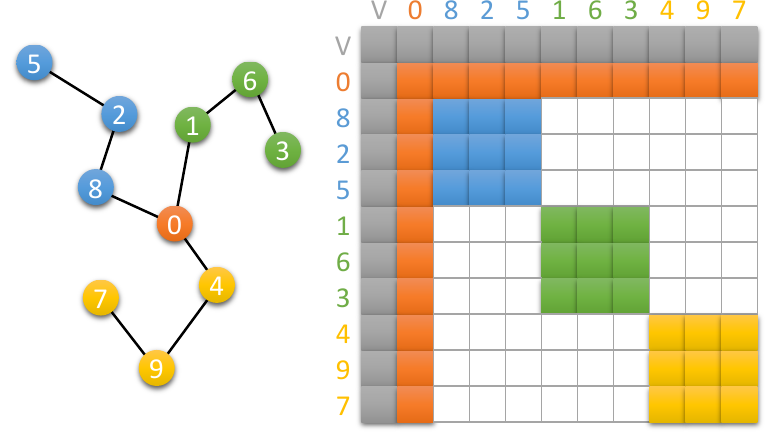}
    \caption{Attention mask.}
    \label{fig:mask}
\end{minipage}
\hfill
\begin{minipage}{0.48\linewidth}
    \centering
    \includegraphics[width=\linewidth]{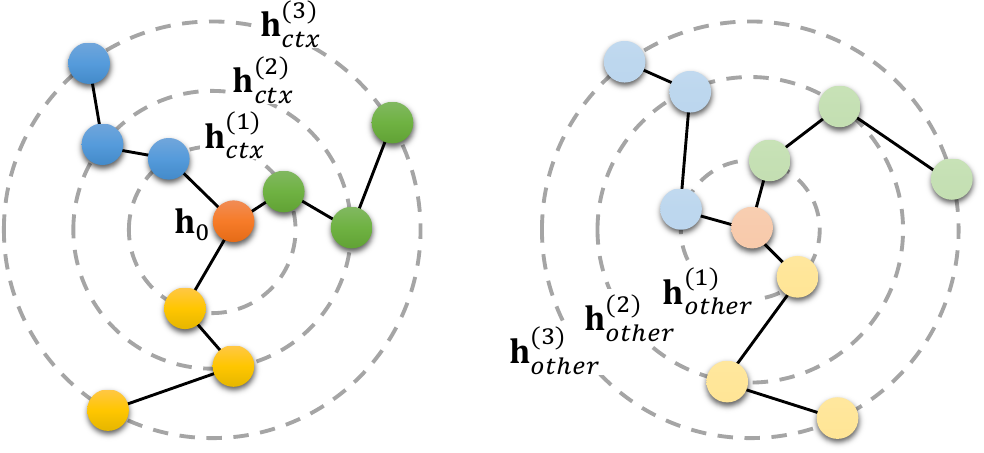}
    \caption{Context learning.}
    \label{fig:context}
\end{minipage}
\vskip -5pt
\end{figure}

In our case, we use a per-walk attention mask to improve scalability. See Figure~\ref{fig:mask} for an illustration. Each random walk will attend to itself but not each other. The virtual token and the root node will still attend to all the tokens. Clearly, with such a mask, the number of nonzeros is reduced by nearly a factor of $k$ and so is the computational cost.

\subsection{Self-supervised pre-training: context prediction}\label{sec:loss}
Let the Transformer output a sequence of vectors corresponding to the input sequence $s(i)$ in~\eqref{eq:seq_input}:
\begin{equation}\label{eq:output}
  [\vh_v, \vh_0, \vh_1^1,\ldots,\vh_{\ell}^1, \vh_1^2,\ldots,\vh_{\ell}^2, \ldots, \vh_1^k,\ldots,\vh_{\ell}^k ].
\end{equation}
The vector $\vh_0$ is the representation of the root node $i$.

The pre-training makes use of the output sequence~\eqref{eq:output} but not task labels. In contrast to the next-token prediction in LLMs, self-supervised learning of GNNs is more often done in a contrastive manner. We follow the infomax principle~\citep{Hjelm2019,Gutmann2012} and develop a contrastive loss suitable for random walks.

The idea is to define increasingly large context windows for the root node $i$ (see Figure~\ref{fig:context}):
\begin{equation}
  \vh_{ctx}^{(j)} = \frac{1}{jk} \sum_{s=1}^j\sum_{r=1}^k \vh_s^r.
\end{equation}
Here, $\vh_{ctx}^{(j)}$ is the representation of the $j$th context window, which includes the nodes up to the $j$th step of all random walks. We maximize the mutual information between the root node $i$ and its context windows, while minimizing that between $i$ and the context windows of other root nodes in the batch. We use an MLP to parameterize the mutual information, leading to the sample loss formula
\begin{equation}\label{eqn:context.loss}
  \tL_{sample} = -\frac{1}{\ell} \sum_{j=1}^{\ell} \Big( \log\text{MLP}( \vh_0\odot\vh_{ctx}^{(j)} )
  + \sum_{\forall other} \log(1-\text{MLP}( \vh_0\odot\vh_{other}^{(j)} )) \Big).
\end{equation}
This loss encourages the representation of the root node to be close to its neighborhood but different from other neighborhoods.

\textbf{Dataset mixture.} Because \model is pre-trained with multiple graph datasets, which may vary in size, we introduce a multiplier $\alpha_D$ for each dataset $D$ of size $n_D$. The batch training consists of multiple passes, each of which iterates over all datasets. For each dataset, a total of $\alpha_Dn_D$ nodes are randomly sampled to form batches.

\subsection{Downstream adaptation}
With the heavy lifting done in pre-training, a downstream use of the model only trains a task head while freezing the pre-trained model. For example, for node classification, the task head is an MLP that takes the node representation as input and outputs the class logits; for link prediction, the input to the MLP task head is a concatenation of the two node representations; and for graph classification, the input is the aggregation of node representations (see Section~\ref{appendix:task_head}). Note that the downstream adaptation can be done for many tasks even though they may seem ``open-ended,'' as long as there is a classification/regression formulation. For example, to predict the shortest path between $u$ and $v$, it suffices to use a prediction head to classify, given any node $w$ on the graph, whether the concatenation of $u$, $v$, and $w$ embeddings will be classified as on the path or off the path~\citep{Battaglia2018}.

Following the practice of LLMs, other adaptation approaches are exploitable, such as fine-tuning all the parameters of the pre-trained model, fine-tuning only the node and edge feature projectors, and using low-rank adaptors or other adaptors, but a full comparison of the different approaches is out of the scope of this work. We find that a simple task head works well.

\section{Theoretical analysis}\label{sec:thm}
The main idea of this work is to use random walks for node representation learning. These walks, together with the SP distance positional encoding, allows reconstructing neighborhoods and distinguishing them. Hence, they are well justified to be part of a foundation model. We formalize these arguments in what follows and provide the accompanying proofs in Section~\ref{appendix:proof}. A graph is denoted as $G(V,E)$ with the node set $V$ and edge set $E$.

\begin{definition}
  The \emph{Shortest Path Distance Oracle} (\emph{SP oracle}) is a function $\psi : V\times V \to \mathbb{R}$ that takes a pair of nodes as input and returns the shortest path distance between these two nodes.
\end{definition}

\begin{definition}
  Denote by $\tB_{u,r} \subset G$ a \emph{ball} centered at node $u$ with radius $r$. Formally, it is a subgraph of $G$ with the node set $V(\tB_{u,r}) := \{v\in V ~|~ \psi(u,v) \leq r\}$ and the edge set $E(\tB_{u,r}) := \{e\in E ~|~ V(e) \subset V(\tB_{u,r})\}$, where $\psi(\cdot,\cdot)$ is the SP oracle.
\end{definition}

\begin{definition}[\citep{Grover2016}]
  A \emph{Biased Random Walk with parameters $p$ and $q$} is a random walk such that after transitioning from node $u$ to node $v$, the unnormalized probability to return to $u$ is $1/p$, that to jump to a direct neighbor of $u$ is 1, and that to jump to other neighbors of $v$ is $1/q$. The usual (unbiased) random walk is recovered with $p=q=1$.
\end{definition}

The following theorem states that a ball can be reconstructed by a sufficient number of random walks together with the SP distance positional encoding.

\begin{theorem}\label{thm:reconstruct}
  Assume that the graph $G$ is undirected and connected, with a bounded degree $d$. Let a ball $\tB_{u,r}$ with center $u$ and radius $r$ have $n$ nodes. The ball can be fully reconstructed given the sequence in Eq.~\eqref{eq:catseq} together with the SP distance of every node from the root $u=i_0$, if the number of walks $k = \Theta(\max(nr, n^2 / r^2))$ and the walk length $\ell = \Theta(r)$.
\end{theorem}

The above theorem derives the complexities of $k$ and $\ell$ by using a biased random walk. The ball can also be reconstructed by using unbiased walks, at the cost of larger $k$ and $\ell$. Empirically, small $k$ and $\ell$ already deliver competitive downstream performance (see the following section for details).

Because random walks can reconstruct a ball, two balls can be distinguished with a graph kernel.

\begin{theorem}\label{thm:express}
  There exists a positive definite kernel function that distinguishes non-isomorphic balls centered at different nodes of $G$.
\end{theorem}

\section{Experiments}\label{sec:exp}
In this section, we present a comprehensive set of experiments to evaluate the effectiveness of \model as a graph foundation model, highlighting transferability in cross-domain and cross-task settings.

\subsection{Experiment setup}\label{sec:exp_setup}
\textbf{Datasets.} We use 14 datasets from diverse domains and for varying tasks. They include those supporting node-level tasks (Cora, CiteSeer, PubMed, Arxiv, WikiCS, and Products, where the first four are \textbf{citation networks} and the next two are the \textbf{Web graph} and the \textbf{co-purchase graph}, respectively); those supporting link-level tasks (WN18RR and FB15k237, which are \textbf{knowledge graphs}); and those supporting graph-level tasks (HIV, PCBA, ChEMBL, and Tox21, which are \textbf{molecules}). We also include Peptides-func and Peptides-struct (also molecules) from the \textbf{Long Range} Graph Benchmark~\citep{Dwivedi2022}. Altogether, these datasets contain 25M nodes and 31M edges. See Section~\ref{appendix:dataset} for more details.

\textbf{Baselines.} We compare \model with ten methods in diverse nature, including PRODIGY~\citep{Huang2023}, OFA~\citep{Liu2024}, and GFT~\citep{Wang2024}, \textbf{which are foundation-model style of methods}; GCN~\citep{Kipf2017}, GIN~\citep{Xu2019}, and GAT~\citep{Velickovic2018}, \textbf{which are GNNs trained in a (semi-)supervised manner}; and DGI~\citep{Velickovic2019}, BGRL~\citep{Thakoor2022}, GraphMAE~\citep{Hou2022}, and GIANT~\citep{Chien2022}, \textbf{which are self-supervised training methods}. Note that OFA differs from foundation models in the usual sense in that it does not have a label-free pre-training stage, but we categorize it together with PRODIGY and GFT to distinguish it from the remaining methods that train a different model for each dataset.

\textbf{Settings.} Our Transformer backbone follows a standard architecture like GPT-2 \citep{gpt2},
with modifications introduced in Section \ref{sec:methodology} and hyperparameters detailed in Section~\ref{appendix:exp.details}. We utilize Llama2-7b-hf~\citep{touvron2023llama} for feature extraction; the prompts can be found in Section~\ref{appendix:feature}. 
All experiments are conducted with 2x NVIDIA Tesla V100 16GB GPUs, Intel Xeon Platinum 8260 CPUs (32 cores), 50GiB RAM, and 1TB user storage space.
Each run is repeated ten times with random seeds.

\begin{table}[t]
  \caption{
  Performance comparison of (semi-)supervised, self-supervised, and foundation-model methods for various domains and tasks. \textbf{Bold} and \ul{underline} highlight the best and sub-best performance. Baseline results are replicated from \cite{Wang2024}.
  }
  \label{tab:performance}
  \vskip 5pt
  \centering
  \small
  \sc
  \resizebox{\linewidth}{!}{
  \begin{tabular}{l|cccc|cc|cc}
    \toprule
        & \multicolumn{4}{c|}{Node classification} & \multicolumn{2}{c|}{Link prediction} & \multicolumn{2}{c}{Graph classi.} \\
        & Cora & PubMed & Arxiv & WikiCS & WN18RR & FB15k237 & HIV & PCBA \\
    \midrule
        MLP & 58.03 & 68.66 & 66.50 & 70.36 & 78.50 & 87.39 & 66.37 & 72.30 \\
        GCN & 75.65 & \ul{75.61} & 71.40 & 75.28 & 73.79 & 82.22 & 64.84 & 71.32 \\
        GIN & 73.59 & 69.51 & 65.05 & 49.77 & 74.02 & 83.21 & 66.86 & 70.12 \\
        GAT & 76.24 & 74.86 & 70.87 & 76.78 & 80.16 & 88.93 & 65.54 & 72.69 \\
    \midrule
        DGI & 72.10 & 73.13 & 69.15 & 75.32 & 75.75 & 81.34 & 59.62 & 63.31 \\
        BGRL & 71.20 & 75.29 & 71.19 & 76.53 & 75.44 & 80.66 & 63.95 & 67.09 \\
        GraphMAE & 73.10 & 74.32 & 70.90 & 77.61 & 78.99 & 85.30 & 61.04 & 63.30 \\
        GIANT & 75.13 & 72.31 & 70.10 & 76.56 & 84.36 & 87.45 & 65.44 & 61.49 \\
    \midrule
        GFT & \ul{78.62} & \textbf{77.19} & \ul{71.93} & \ul{79.39} & \ul{91.91} & \ul{89.72} & \ul{72.67} & \ul{77.90} \\
        \model & \textbf{79.30} & 74.97 & \textbf{75.14} & \textbf{80.27} & \textbf{95.25} & \textbf{95.23} & \textbf{75.15} & \textbf{81.03} \\
    \bottomrule
    \end{tabular}}
\end{table}

\subsection{Cross-domain and cross-task performance}
We first compare \model with a wide array of methods across domains and tasks (node-level, link-level, and graph-level). These methods include (semi-)supervised, self-supervised, and foundation-model methods. The first two classes of methods are not comparable to \model, because they are trained on an individual dataset; however, they set an expectation of the performance. Following GFT, we pre-train \model with ten datasets (see Section~\ref{appendix:exp.details} for their batching ratio) and fine-tune it on each task.
From Table~\ref{tab:performance}, we see that foundation models are uniformly better than individually trained models. Moreover, among foundation models, \model outperforms GFT on seven out of eight tasks.

\subsection{Transferability}
While the outperformance of \model over individually trained models is not surprising, we investigate in depth its transferability. For this, we consider transfer learning and few-shot learning.

\textbf{Transfer learning (dataset- and domain-level transfer).}
This learning paradigm tests a pre-trained model on an unseen dataset or domain. To this end, we pre-train \model with limited datasets and evaluate it on others. Specifically, we use either Arxiv, FB15k237, ChEMBL, or a combination of them and compare with the early use of ten datasets. The three datasets represent different domains: citation, knowledge graph, and molecules.

The results are summarized in Table~\ref{tab:transfer_1} and Figure~\ref{fig:transfer}. We see that using three datasets to pre-train achieves a performance very close to using ten. This suggests that a small amount of representative datasets are already competitive for pre-training, demonstrating the transferability of \model to new datasets. More encouragingly, a model pre-trained with only one dataset achieves competitive performance as well: this performance is significantly higher than the individually trained models on nearly all tasks. For example, the model pre-trained on Arxiv (citation; node classification) performs the second best on WN18RR (knowledge graph; link prediction), with the attained accuracy more than 10 points higher than individually trained models.

To better highlight domain transfer, we plot in Figure~\ref{fig:transfer} the aggregated results of OOD (out-of-domain) and ID (in-domain) performance. Corresponding to Table~\ref{tab:transfer_1}, OOD means, for example, pre-training with Arxiv and testing on knowledge graphs and molecules, while ID means testing on citation networks or web graphs. The fact that the OOD performance is so close to ID, compared with the best of baselines, confirms that our pre-training model and method deliver strong transfer. In particular, random walk patterns enable effective cross-domain generalization.

\begin{table}[t]
  \caption{Transfer learning performance. \dag\ denotes the best performance among all (semi-)supervised methods in Table~\ref{tab:performance}; \ddag\ denotes the best performance among all self-supervised methods in Table~\ref{tab:performance}; * denotes pre-training \model with Arxiv + FB15k237 + ChEMBL.}
  \label{tab:transfer_1}
  \vskip 5pt
  \centering
  \small
  \sc
  \resizebox{\linewidth}{!}{
  \begin{tabular}{l|cccc|ccc|ccc}
    \toprule
        & Cora & Arxiv & WikiCS & (avg.) & WN18RR & FB15k237 & (avg.) & HIV & PCBA & (avg.) \\
    \midrule
        (Semi-)Supervised$^\dag$ & 76.24 & 71.40 & 76.78 & 75.01 & 80.16 & 88.93 & 84.55 & 66.86 & 72.69 & 69.78 \\
        Self-Supervised$^\ddag$ & 75.13 & 71.19 & 77.61 & 74.81 & 84.36 & 87.45 & 85.91 & 65.44 & 67.09 & 66.27 \\
    \midrule
        Pre-Train w/ Arxiv & 76.63 & \ul{75.07} & 79.95 & 76.00 & \ul{94.62} & 94.08 & 94.35 & 73.46 & 79.69 & 76.58 \\
        Pre-Train w/ FB15k237 & 73.71 & 74.40 & 79.50 & 75.29 & 94.34 & 94.82 & 94.58 & 73.57 & 78.95 & 76.26 \\
        Pre-Train w/ ChEMBL & 70.34 & 74.80 & 79.27 & 74.28 & 93.97 & 93.46 & 93.72 & \ul{74.85} & 80.58 & \ul{77.72} \\
        Pre-Train w/ three$^*$ & \ul{78.69} & \ul{75.07} & \ul{80.15} & \ul{76.57} & 94.53 & \ul{95.16} & \ul{94.85} & 74.03 & \textbf{81.14} & 77.59 \\
        Pre-Train w/ all & \textbf{79.30} & \textbf{75.14} & \textbf{80.27} & \textbf{77.42} & \textbf{95.25} & \textbf{95.23} & \textbf{95.24} & \textbf{75.15} & \ul{81.03} & \textbf{78.09} \\
    \bottomrule
    \end{tabular}}
\end{table}

\begin{figure}[t]
\begin{minipage}{0.57\linewidth}
    \centering
    \includegraphics[width=\linewidth]{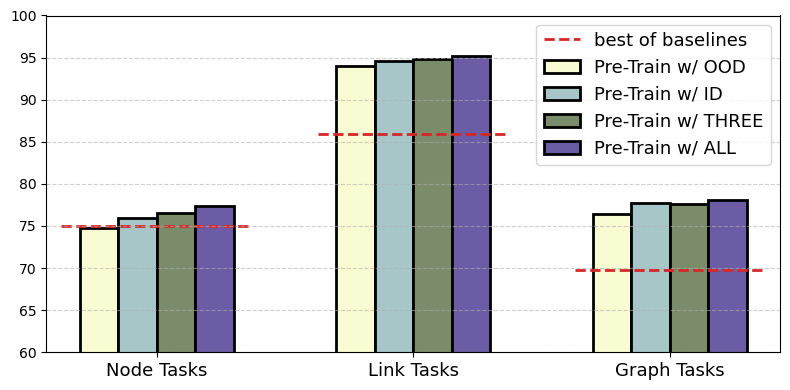}
\end{minipage}\hfill
\begin{minipage}{0.42\linewidth}
    \caption{
    Aggregated transfer learning performance. ``Best of baselines'' denotes the highest score among (semi-)supervised and self-supervised methods. ``OOD'' (resp. ``ID'') indicate that the datasets for pre-training and downstream testing are from the same (resp. different) domain. ``Pre-Train w/ THREE'' and ``Pre-Train w/ ALL'' follow the definitions in Table~\ref{tab:transfer_1}.
    }
    \label{fig:transfer}
\end{minipage}
\end{figure}

\begin{table}[t]
  \caption{
  Few-shot learning performance. Results of BGRL, GraphMAE, GIANT, PRODIGY, OFA, and GFT are replicated from \cite{Wang2024}.
  }
  \label{tab:fewshot_1}
  \vskip 5pt
  \centering
  \small
  \sc
  \setlength{\tabcolsep}{4pt}
  \resizebox{\linewidth}{!}{
  \begin{tabular}{l|ccc|ccc|ccc|ccc}
    \toprule
        & \multicolumn{3}{c|}{Arxiv 5-way} & \multicolumn{3}{c|}{Arxiv 40-way} & \multicolumn{3}{c|}{FB15k237 10-way} & \multicolumn{3}{c}{FB15k237 40-way} \\
        & 1-shot & 3-shot & 5-shot & 1-shot & 3-shot & 5-shot & 1-shot & 3-shot & 5-shot & 1-shot & 3-shot & 5-shot \\
    \midrule
        MLP & 36.53 & 41.40 & 44.05 & 8.63 & 13.75 & 16.64 & 49.96 & 59.00 & 67.10 & 37.66 & 43.82 & 49.44 \\
        GCN & 52.60 & 57.8 & 64.52 & 19.97 & 29.89 & 34.82 & 69.14 & 85.02 & \ul{91.66} & 55.95 & 69.38 & 73.90 \\
        GAT & 46.60 & 64.08 & \textbf{75.10} & 23.68 & 33.51 & \ul{38.63} & 66.14 & 84.05 & 88.67 & 52.73 & 72.24 & 73.75 \\
        GIN & 31.28 & 40.30 & 42.20 & 12.80 & 13.75 & 14.52 & 76.16 & \ul{88.90} & 91.22 & 66.81 & \ul{77.61} & \ul{79.43} \\
    \midrule
        DGI & 40.07 & 46.73 & 50.67 & 11.75 & 15.06 & 18.24 & 70.93 & 72.37 & 85.47 & 59.37 & 63.41 & 66.68 \\
        BGRL & - & 48.43 & - & - & 17.98 & - & - & 67.23 & - & - & 29.24 & - \\
        GraphMAE & - & 49.24 & - & - & 19.12 & - & - & 69.75 & - & - & 32.07 & - \\
        GIANT & - & 54.33 & - & - & 20.12 & - & - & 77.21 & - & - & 52.63 & - \\
    \midrule
        PRODIGY & 48.23 & 58.64 & 61.09 & 21.44 & 23.69 & 25.51 & 66.10 & 79.61 & 84.30 & 54.30 & 59.58 & 62.03 \\
        OFA & \ul{52.80} & 58.68 & 59.92 & 21.34 & 22.13 & 24.01 & 83.46 & 83.14 & 83.64 & 63.48 & 65.76 & 66.51 \\
        GFT & \textbf{58.20} & \ul{66.00} & 68.00 & \ul{26.49} & \ul{34.36} & 36.29 & \ul{88.07} & 88.53 & 89.13 & \ul{74.97} & 74.56 & 75.01 \\
    \midrule
        \model & 52.76 & \textbf{72.06} & \ul{73.58} & \textbf{26.72} & \textbf{39.72} & \textbf{43.14} & \textbf{93.24} & \textbf{94.34} & \textbf{95.28} & \textbf{82.45} & \textbf{88.85} & \textbf{90.67} \\
    \bottomrule
    \end{tabular}}
\end{table}

\textbf{Few-shot learning (label-level transfer).}
In this learning paradigm, a support set of $N$ classes with $k$ examples each is given and one is asked about the class of the query. Typically, $k$ is very small and the classes used for testing are not seen in training. Hence, few-shot learning tests two abilities: the ability to predict new labels and the ability to learn from very few examples.

In Table~\ref{tab:fewshot_1}, we fine-tune \model on a few select datasets, for each of which we conduct a few $N$-way $k$-shot experiments. (See Section~\ref{appendix:task_head} for fine-tuning details.) We compare \model with methods that train a separate model for each dataset and methods that (pre-)train a single model; i.e., foundation-model-style methods.
From Table~\ref{tab:fewshot_1}, we see that foundation-model-style methods outperform individually trained models, which is not surprising due to the limited training examples. Moreover, our model \model performs the best in nearly all cases, with the second best generally attained by GFT.

While it is uncommon to include supervised and self-supervised baselines in few-shot settings, we follow GFT \citep{Wang2024} by fine-tuning models on limited samples to enable this comparison. MLP, DGI, and our method share the same prediction head, differing only in whether using raw or LLM-generated features. GNN baselines rely on raw features but incorporate structural information via message passing. The results highlight that both feature quality and structural awareness are crucial for few-shot learning, supporting our claim that our model’s generated embeddings effectively integrate the graph structure and dataset context.

\subsection{Ablation study}\label{sec:exp.ablation}

\textbf{Comparison between random-walk sampling and neighborhood sampling.}
We motivate the use of random walks for representation learning in Section~\ref{sec:intro} with a few reasons. One of the reasons is that random walks can better cope with long-range interactions if they are important for some downstream tasks or datasets. Compared with neighborhood sampling~\citep{Hamilton2017}, multiple random walks equally retain sufficient near-hop neighbors while being able to extend very far.

To substantiate this argument, we perform an experiment to compare the two sampling approaches. We reuse the Transformer backbone but prepare the input sequences differently. For neighborhood sampling, we list all sampled nodes in a sequence in an arbitrary order. The ordering does not matter because we use the hop number as the positional encoding. Note that it is impossible to form edge features equivalent to the random-walk case because adjacent nodes in the sequence are not necessarily connected by an edge. Hence, we ignore edge features for neighborhood sampling. We compare the two approaches by using a similar sequence length.

From Table~\ref{tab:sampler} in Section~\ref{app:ablation.walk.vs.neighbor}, we do not see one setting that performs uniformly the best, but we note that using a walk length $\ell=8$ achieves the best result for the two datasets from the Long Range Graph Benchmark (Peptides-func and Peptides-struct). 
Moreover, we see that random-walk sampling generally achieves better results than neighborhood sampling across datasets. Ultimately, the best walk length is dataset-dependent, but random walks always offer an opportunity to capture a larger receptive field, if ever needed.

\textbf{Comparison of training losses.}
We propose a new loss~\eqref{eqn:context.loss} in Section~\ref{sec:loss} to pre-train our Transformer. This pre-training loss contrasts the context of a node and those of other nodes. We compare this loss with other popular contrastive losses in the graph literature: DGI~\citep{Velickovic2019}, GraphPrompt~\citep{graphprompt}, and MaskGAE~\citep{maskgae}. In brief, DGI uses a random sequence as the negative context; GraphPrompt uses an InfoNCE-style of loss, which sums over all contexts in the denominator; and MaskGAE contrasts the next-hop contexts.
Additionally, we compare our loss with the mask-token prediction approach commonly used to pre-train LLM encoders. Specifically, we add a token (or position) reconstruction term to the loss.
See Sections~\ref{appendix:ctx_loss} and~\ref{appendix:mask_loss} for the mathematical details.

The results are reported in Tables~\ref{tab:ctx_loss_1} and~\ref{tab:mask_loss} of Sections~\ref{appendix:ctx_loss} and~\ref{appendix:mask_loss}, respectively. Compared with other contrastive losses, our loss achieves the best results on 5 out of 9 datasets. No other compared losses perform the best as many. Meanwhile, adding a reconstruction term flips four best cases to second best while flipping four second best to best. We conclude that our loss is more favorable than other contrastive losses and adding a reconstruction term is not beneficial.

\section{Conclusions}\label{sec:conclusion}
We have developed a Transformer approach toward building a foundation model for graphs. Central to this approach is the use of random walks as input sequences and the setup of context prediction for self-supervised learning, allowing pre-training with multiple datasets from different domains. The resulting model, \model, can be adapted to a new graph and task in a new domain. With comprehensive experiments, we show that \model is as competitive as graph models trained on a single dataset and those trained on multiple datasets in a supervised manner.

\textbf{Limitations.} While \model shows promise in various aspects of a foundation model (e.g., diverse pre-training data, downstream adaptation, and out-of-dataset and domain transfer), as a research prototype it has not reached the scale of the counterpart models for natural language data, which have several to several hundred billion parameters and consume billions to trillions of training tokens. Nevertheless, this study paves the way toward matching the scale of successful LLMs.

\textbf{Broader impacts.} Graphs as a unique kind of data embedded in every aspect of lives, ranging from mathematical reasoning to social bonding. Foundation models for graphs are complementary to those for natural languages and their technological advancement benefits the human society.

\newpage
\bibliography{reference}
\bibliographystyle{plain}

\newpage
\appendix
\onecolumn
\tableofcontents
\part*{Appendix}

\section{More related work}\label{appendix:rw}

\textbf{Graph transformers.}
In addition to the GTs predominantly appearing in the graph literature (Section \ref{sec:intro}), the NLP community also explored Transformer-based architectures for structured data. 
G2GTr \citep{mohammadshahi2019graph} and RNGTr \citep{mohammadshahi2021recursive} adapt Transformers to handle graph-structured input and output in dependency parsing tasks. 
Although not originally designed for general graphs, these models incorporate structure-aware mechanisms that capture local connectivity within sequential input, typically consisting of redundant neighborhood nodes. Our method shares a similar intuition with these methods.

\textbf{Pre-training GNNs.}
Pre-training is essential to LLMs and to a degree, GNNs. An early GNN pre-training study~\citep{Hu2020} develops a multitude of loss terms that, when invoked together, incur positive transfer and significantly improve task performance, while naive pre-training methods incur negative transfer. However, some of the loss terms exploit task labels (e.g., graph property prediction). Nevertheless, another loss term---context prediction---is particularly useful and our method is inspired by this idea. Additionally, a proliferation of self-supervised learning methods were developed for GNNs in the past~\citep{Velickovic2019,Zhu2020,You2020,Zhu2021,Thakoor2022,Hou2022,Chien2022}. They could be used for pre-training and we compare some of them in the experiments.

\textbf{Foundation model concepts in GNNs.}
LLMs bring in new concepts that motivate the modern development of GNNs, even if the objective is not a foundation model. Prompting is one example. 
GPPT~\citep{Sun2022} pre-trains a usual GNN and introduces a prompting function to adapt the GNN to produce (node, label) embedding pairs for node classification. 
GraphPrompt~\citep{graphprompt} incorporates learnable prompts and a task-unified objective to bridge the gap between pre-training and fine-tuning in GNNs, enabling the model to extract task-relevant information from a pre-trained model. 
PRODIGY~\citep{Huang2023} focuses on the few-shot learning setup, which builds a graph that connects a few prompt examples (the shots), the query, and the task labels for classifying the query.
ULTRA~\citep{galkin2023towards}, designed for knowledge graphs, leverages relative entity and relation embeddings for inductive and transductive reasoning. In contrast, our approach does not pre-train on relations in link-prediction datasets, because relations are inherently tied to the link-prediction downstream task. Instead, we concatenate head and tail node embeddings to predict the relation—a simpler yet effective strategy.

Several recent works aim to unify the feature space across graph datasets. 
All-in-One~\citep{Sun2023} trains prompt tokens derived from subgraphs and inserts them into node representations, enabling a pre-trained GNN to adapt to multiple tasks. 
Zhao et al.~\citep{zhao2024all} propose a prompt-based method for cross-domain classification, where features are aligned through singular value decompositions. 
OpenGraph~\citep{xia2024opengraph} introduces a topology-aware graph tokenizer that converts graphs into sequences of unified embeddings, which are then processed by a scalable graph transformer. In contrast, our model aligns datasets using a pre-trained LLM without requiring a dataset-specific tokenizer—though such tokenizers remain compatible for task-specific adaptations.
For text-attributed graphs, OFA~\citep{Liu2024}, GFT~\citep{Wang2024}, and UniGraph~\citep{unigraph} leverage pre-trained LLMs to unify textual features. We adopt a similar strategy to bridge cross-domain graph representations. 
Additionally, inspired by GCC~\citep{qiu2020gcc}, we incorporate contrastive learning on subgraph patterns to enhance cross-task discrimination.

\textbf{LLM-based methods.}
Several LLM-centric approaches have been proposed for graph tasks, where graph data is converted into text-compatible formats that LLMs can process. 
GraphPrompter~\citep{Liu2024a} and GNP~\citep{Tian2024} integrate GNN embeddings as soft prompts into text tokens, enabling LLMs to handle graph-structured inputs. 
Other studies~\citep{Wang2023, Fatemi2024} show that LLMs can capture structural properties through natural language descriptions of nodes and their connections. 
GPT4Graph~\citep{Guo2023} assesses the ability of LLMs to understand and reason over structured graph data. 
GraphLLM~\citep{graphllm} augments LLMs with a graph transformer to improve graph reasoning. 
GraphText~\citep{graphtext} transforms graph data into natural language sequences using graph-syntax trees, allowing in-context reasoning without fine-tuning. 
GraphGPT~\citep{graphgpt} aligns graph and textual representations via dual-stage instruction tuning, improving both supervised and zero-shot performance. 
LLaGA~\citep{llaga} translates graphs into LLM-readable formats to support prompt-based task execution.
In contrast, our model employs a lightweight task-specific prediction head without fine-tuning the LLM, resulting in a compact encoder comparable in size to GPT2-124M. 
Similarly, WalkLM~\citep{walklm} uses textual random walks to construct LLM input and contrastive pre-training with task-specific heads; however, it does not support joint training across domains.

Readers can refer to \cite{liu2023towards} for a more detailed survey on graph foundation models.

\section{Downstream adaptation}
\label{appendix:task_head}
The pre-trained model needs to be adapted to perform a downstream task. In this work, we take a simple form of adaptation by using task heads. Specifically, a task head is a neural network that takes the pre-trained model output as input and performs the downstream task. Because we perform evaluation in multiple scenarios (cross-domain cross-task learning, transfer learning, and few-shot learning), the task heads include some of the following components: encoder, decoder, and scorer.
\begin{itemize}[leftmargin=*]
\item The \textbf{encoder} maps a generic node representation (the pre-trained model output $\vh_0$ in~\eqref{eq:output}) to a task-specific node representation. The encoder is an MLP, where each layer consists of four components: linear transformation, batch normalization, ReLU activation, and dropout. 
For node-level tasks, the input consists of node representations. For link-level tasks, the input is the concatenation of the representations of the two nodes forming the edge. For graph-level tasks, the input graph representations is computed by aggregating node representations across the entire graph. Hyperparameters are detailed in Section~\ref{appendix:exp.details}.

\item The \textbf{decoder} maps task-specific node representations to task targets.
Note that concatenation for link tasks and pooling for graph tasks are performed within the encoder. Therefore, the decoder is always a single linear layer, where the inputs are the outputs of the encoder. The output dimension of the decoder corresponds to the number of classes for classification and a single digit for regression.

\item The \textbf{scorer} is an unparameterized module responsible to compare two representations. In few-shot learning, this scorer will compare the query-node representation and the class representation. We use the cosine similarity for comparison and apply a softmax on the vector of similarities for label prediction. The class representation is computed as the average of the examples in the support set.
\end{itemize}

For the cross-domain cross-task learning and transfer learning, the task head is the encoder--decoder pair. Each task head is trained by using the training set of the corresponding task with labels.

For few-shot learning, we first train the encoder--decoder pair by using very few examples. Then, we use the trained encoder and pair it with the scorer for prediction. It was suggested that this approach is better than directly training the encoder--scorer pair~\citep{Huang2023,Wang2024} and we confirm so in preliminary tests.

\section{Proofs for Section~\ref{sec:thm}}\label{appendix:proof}

\subsection{Proof of Theorem \ref{thm:reconstruct}}\label{proof:reconstruct}

\begin{lemma}
\label{lemma:walk_length}
    The expected walk length for reaching the farthest node in $B_{u,r}$ is $\expect[\ell] = O(r)$ when
    \begin{equation}
    \label{eq:lemma_walk_length}
        \frac{1}{q} \geq \frac{C+1}{C-1}\cdot\frac{1}{p} + \frac{1}{C-1} d,
    \end{equation}
    for some constant $C \geq 2$ and $C \ll r$.
\end{lemma}
\begin{proof}
    Let us consider three consecutive nodes in a walk $(v_0,v_1,v_2)$ with $\psi(u,v_0) = i-1$ and $\psi(u,v_1) = i$. For the transition from $v_1$ to $v_2$, we assume the probability of moving back, i.e., $v_2 = v_0$, is $p_b$; the probability of moving outward to a node with $\psi(u,v_2) = i+1$ is $p_f$; and the probability of moving in the same depth is $1-p_b-p_f$. 
    
    Following Lemma 4.7 in \cite{blum2015foundations}, we denote $h_{i,i+1}$ as the hitting time from a node of distance $i$ to a node of distance $i+1$. Then,
    \begin{equation}
        h_{i,i+1} = p_f \cdot 1 + p_b (1 + h_{i-1,i} + h_{i,i+1}) + (1-p_f-p_b)(1 + h_{i,i+1}).
    \end{equation}
    Since the initial state $h_{0,1} = 1$, solving above equation yields the following recursion:
    \begin{equation}\label{eq:hitting_per_step}
        h_{i,i+1} = \frac{1 + p_b h_{i-1,i}}{p_f} = \left(\frac{p_b}{p_f}\right)^i + \frac{1}{p_b} \sum_{t=1}^{i} \left(\frac{p_b}{p_f}\right)^t.
    \end{equation}

    Consider the worst case of the biased random walk under the degree bound assumption, where we have $p_f = p/(p+q+pqd)$ and $p_b = q/(p+q+pqd)$. Assumption \eqref{eq:lemma_walk_length} implies $p_f > p_d$; therefore, we can define $\alpha := p_b/p_f = q/b < 1$. Then,
    \begin{equation}
        h_{0,r} = \sum_{i=0}^{r-1} h_{i,i+1} 
        = (1 - \frac{1}{p_f - p_b})\frac{1 - \alpha^r}{1 - \alpha} + \frac{r}{p_f - p_b}.
    \end{equation}
    Since $1 > p_f > p_b > 0$, the first term is negative, we only need to consider the second term. From Eq. \eqref{eq:lemma_walk_length}, we have $p_f - p_b \geq 1/C$, which leads to:
    \begin{equation}
        h_{0,r} \leq \frac{r}{p_f - p_b} \leq C\cdot r.
    \end{equation}
\end{proof}

The above lemma indicates that the expected number of random walk steps to traverse from the root node to a node with distance $r$ is $O(r)$, with proper choices of $p$ and $q$.

\begin{corollary}
    When $p = q$, the expected walk length for reaching the farthest node in $\tB_{u,r}$ is $\expect[\ell]=\Theta(r^2)$.
\end{corollary}
\begin{proof}
    Starting from Eq.~\eqref{eq:hitting_per_step} with $p_b/p_f = 1$, we obtain
    \begin{equation}
        h_{i,i+1} = 1 + \frac{i}{p_b}
        \quad\to\quad
        h_{0,r}
        = \sum_{i=0}^{r-1} h_{i,i+1}
        = r + \frac{1}{p_b}\frac{r(r-1)}{2} 
        = \Theta(r^2).
    \end{equation}
\end{proof}

The above corollary indicates that for an unbiased random walk, the expected walk length $\ell$ is of the same order as Lemma 4.7 of \cite{blum2015foundations}, up to a constant factor.

\begin{corollary}
    When $p < q$, the expected walk length for reaching the farthest node in $\tB_{u,r}$ can be exponential in $r$.
\end{corollary}

The above corollary indicates that in rare cases, where a random walk favors moving back over moving outward, the expected walk length can grow extremely fast.

Here are the interpretations of the results so far. (1) When the random walk has equal probability of moving inward or outward, the expected hitting time has the same order as that of an unbiased random walk. (2) When the random walk favors returning rather than progressing, the hitting time can grow exponentially. (3) In most practical cases, where the walk prioritizes depth over breadth and backward steps, the hitting time is linear and efficient to cover enough range.

Empirically, the required walk length $\ell$ indeed grows linearly. For example, the Cora dataset has an average node degree $d \approx 4$. We set $p = 1$ and $q = 0.1$ to favor moving outward over backward. Assume that the desired $r = 7$. Lemma~\ref{lemma:walk_length} suggests a walk length of $\ell = 12$; meanwhile, we use $\ell = 8$ and observe that nodes at distance 7 are frequently reached.

Before reconstructing $\tB_{u,r}$, we show that the input sequence \eqref{eq:catseq} can provide enough distinct SP oracles $\psi(v_1,v_2)$ within $\tB_{u,r}$.

\begin{lemma}
\label{lemma:num_walks}
    Assume that $\tB_{u,r}$ contains $n$ nodes. Then, the expected number of random walks with $\ell = \Theta(r)$ to cover all nodes in $\tB_{u,r}$ is 
    \begin{equation}
        \expect[k] = O(nr).
    \end{equation}
\end{lemma}
\begin{proof}
    According to Lemma \ref{lemma:walk_length}, a random walk $w$ of length $\ell$ is expected to yield at least $r$ distinct SP oracles $\psi(u,v)$ where $v\in V(w)$ and $u$ is the root node. This provides at least as much information as a depth-first search (DFS) walk of length $r$, ignoring nodes with the same depth.

    Let us consider a DFS tree with root node $u$ and depth $r$. Assume it has $m_r$ nodes.
    Clearly, $m_r \leq \min(n, d^r)$.
    Now, we compute the expected number of random DFS walks of length $r$ required to cover all leaves.
    This is the \emph{coupon collector problem}; hence, $k = \Theta(m_r \ln m_r)$.
    Similarly, to cover all nodes we need $k = \Theta(m \ln m)$ where $m = \max_i(m_i)$. Therefore, $k = O(mr) = O(nr)$.
\end{proof}

\begin{lemma}
\label{lemma:SP_oracle}
    With the same assumption in Lemma \ref{lemma:num_walks}, given $\ell = \Theta(r)$ and $k = \Theta(n^2/r^2)$, the expected number of distinct SP oracles provided by one input sequence \eqref{eq:catseq} is $\Theta(n^2)$.
\end{lemma}
\begin{proof}
    We define the \emph{pseudo SP oracle} given by a random walk $w$ as:
    \begin{equation}
        \widetilde{\psi}_w(v_1,v_2) := |w(v_1) - w(v_2)|,
    \end{equation}
    where $w(v)$ returns the index of node $v$ in walk $w$.
    Then, by denoting $W = \{w_1,\dots,w_k\}$, we further define:
    \begin{equation}
        \widetilde{\psi}(v_1,v_2) := \min_{w \in W} \widetilde{\psi}_w(v_1,v_2).
    \end{equation}
    We need to justify the cases where the pseudo SP oracle is accurate and yet reliable.
    For example, $\widetilde{\psi}(v_1,v_2) = \psi(v_1,v_2)$ if the shortest path between $v_1$ and $v_2$ is traversed by a random walk $w \in W$. Moreover, if $\widetilde{\psi}(v_1,v_2) = 1$, i.e. $v_1, v_2$ are consecutive nodes in a walk, it is also accurate. The smaller $\widetilde{\psi}(v_1,v_2)$ is, the higher the probability that the pseudo SP oracle is accurate. 

    In general, if the nodes in a subsequence $\underline{w}\subset w$ exhibit strictly increasing depth, i.e.,
    \begin{equation}
        \psi(u,v_j) = \psi(u,v_{j-1}) + 1,
        \quad
        \forall~ [v_{j-1}, v_j] \subset \underline{w},
    \end{equation}
    then the pseudo SP oracle yields accurate estimates for any pair of nodes within the subsequence $\underline{w}$.
    Here, $[v_{j-1}, v_j]$ denotes any two consecutive nodes in $\underline{w}$.
    Consequently, the number of reliable SP oracles is $|\underline{w}|(|\underline{w}|-1)/2$.

    We consider a path of $r+1$ nodes with depth from $0$ to $r$. We assume the probability of moving forward is $p_f$, which is equivalent to adding a break point at each internal point with probability $1-p_f$.
    The expected number of strictly increasing subsequences is: $\expect[n_s] = 1 + (r-1)(1-p_f)$. Therefore, the expected number of nodes in a subsequence is $\expect[\ell_s] = 1 + r/\expect[n_s]$. Then, the expected number of SP oracles is:
    \begin{equation}
        \expect[n_{SP}] 
        = \expect[n_s]\frac{\expect[\ell_s](\expect[\ell_s]-1)}{2}
        = \frac{r^2}{2 + 2(r-1)(1-p_f)} + \frac{r}{2}.
    \end{equation}
    Let $p_f = 1 - 1/r$ by tuning the biased random walk. Then we have:
    \begin{equation}
        \expect[n_{SP}] \approx \frac{1}{4}(r+1)^2 = \Theta(r^2),
    \end{equation}
    which leads to the conclusion given the assumptions.
\end{proof}

In practice, $k$ can often be much smaller, particularly when the graph is sparse or far from complete. 
For instance, in a 2D grid where $n = \Theta(r^2)$, we obtain $k = \Theta(r^2) = \Theta(n)$. The sparser the graph, the smaller $k$ needed.
Therefore, as suggested by Lemmas~\ref{lemma:num_walks} and~\ref{lemma:SP_oracle}, the selection of 
$k$ should balance two factors: ensuring sufficient coverage of the graph and enough number of accurate SP oracle estimations.

We are now ready to show the main theorem as a consequence of the following propositions.

\begin{proposition}[{\cite[Table 1]{reyzin2007learning}}]
\label{prop:1}
    The reconstruction of an arbitrary graph requires $\Theta(n^2)$ SP oracles. It can be reduced to $\Theta(dn \log n)$ for trees of bounded degree $d$.
\end{proposition}

\begin{proposition}[{\cite[Theorem 1]{mathieu2013graph}}]
\label{prop:2}
    With a constant degree bound $d$, reconstruction complexity can be reduced to $\Theta(n^{3/2})$ with a randomized method.
\end{proposition}

\begin{theorem}[Theorem \ref{thm:reconstruct}]
    The ball $\tB_{u,r}$, can be reconstructed via the input sequence \eqref{eq:catseq} with $k=\Theta(\max(nr, n^2/r^2))$ and $\ell = \Theta(r)$.
\end{theorem}
\begin{proof}
    Lemmas \ref{lemma:walk_length}, \ref{lemma:num_walks} and \ref{lemma:SP_oracle} guarantee that with appropriate $k$ and $\ell$, the input sequence \eqref{eq:catseq} can fully cover all nodes in the ball and provide $\Theta(n^2)$ SP oracles. Then, following Propositions \ref{prop:1} and \ref{prop:2}, we can reconstruct $\tB_{u,r}$, which concludes the proof.
\end{proof}

\subsection{Proof of Theorem \ref{thm:express}}\label{proof:express}
We now show the expressivity of the input sequence \eqref{eq:catseq}. We abbreviate a ball as $B_i:=B_{u_i,r}$ since the radius does not matter for graph isomorphism. The proof is based on the following proposition.

\begin{proposition}[{\cite[Section 4]{spkernels}}]
\label{prop:3}
    There exists a graph kernel based on SP oracles that is computable in polynomial time, expressive up to isomorphism, and positive definite.
\end{proposition}

\begin{theorem}[Theorem \ref{thm:express}]
    There exists a positive definite kernel function $\Phi$ that can distinguish two balls $\tB_1$ and $\tB_2$ with $u_1 \neq u_2$, up to isomorphism.
\end{theorem}

\begin{proof} 
    According to Theorem~\ref{thm:reconstruct}, the input sequence \eqref{eq:catseq} can recover SP oracles $\psi$ with appropriate $k$ and $\ell$.
    Specifically, we can build a subgraph kernel:
    \begin{equation}
    \label{eq:SP_kernel}
        \Phi(\tB_1,\tB_2) = \sum_{s_1 \in SP(\tB_1)} \sum_{s_2 \in SP(\tB_2)} \phi(s_1, s_2)
    \end{equation}
    where $SP(\tB)$ denotes the set of triplets $(v_1,v_2,\psi(v_1,v_2))$ for all $v_1,v_2\in\tB$, and $\phi$ checks if $s_1 = s_2$. The function $\Phi$ is equivalent to the kernel in Proposition \ref{prop:3} and is therefore positive definite and expressive up to graph isomorphism.
\end{proof}

\section{Cost analysis}\label{appendix:scalability}

A common conception is that Transformers are expensive to train and use. Here, we conduct an analysis to compare the model size and the forward time of \model versus GNNs. Surprisingly, under reasonable choices of the hyperparameters, \model can be more economical than GNNs.

The main hyperparameters of a Transformer are the context length $L$ and the embedding dimension $d$. Each Transformer block contains four $d\times d$ and two $d\times(4d)$ weight matrices; hence, the per-block model size is dominated by $12d^2$. The main computations are matrix-matrix multiplications of size $(s_m,s_k,s_n)$ per sequence: three $(L,d,d)$, two $(L,d,L)$, and two $(L,d,4d)$. Thus, the per-block time cost is dominated by $11Ld^2+2L^2d$. Overall, for a Transformer with $T$ blocks and a batch size $B$, the approximate model size and forward time are summarized in Table~\ref{tab:cost}.

The main hyperparameters of a GNN with neighborhood sampling are depth $\ell'$ and fanouts $k_1,\ldots,k_{\ell'}$. Reusing the batch size $B$, assume that on average the number of sampled nodes for each aggregation is $s_0(=B),s_1,\ldots,s_{\ell'}$. Further, assume that all the weight matrices are $d'\times d'$. Then, the model size of a GNN is approximately $(d')^2\ell'$. In the $(\ell'-i)$th layer, there is a matrix-matrix multiplication of size $(s_i,d',d')$ and an aggregation of $k_i$ neighbors for each of the $s_{i-1}$ nodes. Thus, the per-layer time cost is $s_i(d')^2+s_{i-1}k_id'$. If we set each $k_i=k'$ and consider the worst case $s_i = s_{i-1}k'$, a simplified upper bound of the overall forward time is given in Table~\ref{tab:cost}.

The size and forward time of the two models are not directly comparable, but we can consider a few cases. If they have the same depth ($T=\ell'$), the same embedding dimension ($d=d'$), and if we allow $L=d$ in the Transformer, then the model size of \model is $12$ times of GNN, but \model is faster than GNN as long as $(k')^T > 13dT$. If instead, we increase the embedding dimension of the GNN to $d'=d\sqrt{12}$, the two models have the same size and \model is faster than GNN as long as $(k')^T > \frac{13}{12}dT$.

\begin{table}[h]
  \caption{Cost comparison. $B$: batch size; $T$: number of Transformer blocks; $L$: context length; $d$: Transformer embedding dimension; $\ell'$: number of GNN layers; $d'$: GNN embedding dimension; $k'$: fanout.}
  \label{tab:cost}
  \vskip 5pt
  \centering
  \small
  \sc
  \begin{tabular}{lccccc}
    \toprule
    & Model size & Forward time \\
    \midrule
    \model & $12d^2T$ & $(11Ld^2+2L^2d)TB$ \\
    GNN & $(d')^2\ell'$ & $\frac{k'}{k'-1}(d')^2(k')^{\ell'}B$ \\
    \bottomrule
  \end{tabular}
\end{table}

\section{Datasets}\label{appendix:dataset}

\begin{table}[t]
  \caption{Dataset statistics.}
  \label{tab:dataset}
  \vskip 5pt
  \centering
  \small
  \sc
  \resizebox{\linewidth}{!}{
  \begin{tabular}{lccccccc}
    \toprule
    Dataset & Domain & Task & \#Graphs & Avg. \#Nodes & Avg. \#Edges & \#Task & \#Class \\
    \midrule
    Cora & citation & node & 1 & 2,708 & 5,429 & 1 & 7 \\
    citeseer & citation & node & 1 & 3,312 & 4,598 & 1 & 6 \\
    PubMed & citation & node & 1 & 19,717 & 44,338 & 1 & 3 \\
    Arxiv & citation & node & 1 & 169,343 & 1,166,243 & 1 & 40 \\
    WikiCS & web link & node & 1 & 11,701 & 216,123 & 1 & 10 \\
    Products & co-purchase & node & 1 & 54,025 & 144,638 & 1 & 44 \\
    WN18RR & knowledge & link & 1 & 40,943 & 93,003 & 1 & 11 \\
    FB15k237 & knowledge & link & 1 & 14,541 & 310,116 & 1 & 237 \\
    HIV & molecule & graph & 41,127 & 25.45 & 27.47 & 1 & 2 \\
    PCBA & molecule & graph & 437,929 & 25.97 & 28.10 & 128 & 2 \\
    ChEMBL & molecule & graph & 365,065 & 25.84 & 27.96 & 1,048 & 2  \\
    Tox21 & molecule & graph & 7,831 & 18.61 & 19.34 & 12 & 2 \\
    Peptides-func & molecule & graph & 15,535 &  150.94 & 307.30 & 10 & 2 \\
    Peptides-struct & molecule & graph & 15,535 &  150.94 & 307.30 & 11 & regres. \\
    \bottomrule
  \end{tabular}}
\end{table}

We use a total of 14 datasets for experimentation. Their statistics is summarized in Table~\ref{tab:dataset}.

\textbf{Cora:} Cora~\citep{cora} is a citation network comprising 2,708 machine learning publications, each being a node. These nodes are interconnected by 5,429 edges, indicating the citation relationship between papers. The publications come from seven subfields of machine learning. The paper titles and abstracts are available at \url{https://people.cs.umass.edu/~mccallum/data.html}.

\textbf{citeseer:} citeseer~\citep{citeseer} is also a citation network, comprising 3,312 scientific publications across six disciplines. Each paper's abstract can be obtained from \url{https://people.cs.ksu.edu/~ccaragea/russir14/lectures/citeseer.txt}.

\textbf{PubMed:} PubMed~\citep{yang2016revisiting} consists of 19,717 papers from the PubMed database, focusing on diabetes research. The text data can be obtained from the TAPE repository~\citep{he2023harnessing}.

\textbf{Arxiv:} The Arxiv (short for ogbn-arxiv) dataset is a citation network of Computer Science papers submitted to arXiv. It is part of the OGB collection~\citep{Hu2020a} and the paper titles and abstracts can be obtained from the OGB GitHub repository.

\textbf{Products:} The Products (short for ogbn-products) dataset is also part of the OGB collection~\citep{Hu2020a}. It is an undirected and unweighted graph representing the Amazon product co-purchasing network, containing approximately 2.4 million products and 61.9 million edges. For our experiments, we use the text data from the TAPE repository~\citep{he2023harnessing}, which covers only a subgraph of the original ogbn-products.

\textbf{WikiCS:} WikiCS~\citep{wikics} is a graph constructed from Wikipedia, comprising 11,701 articles as nodes and 216,123 hyperlinks as edges. The node features are derived from the articles' text and it is labeled according to the article's subject category. The text data can be obtained from the original paper.

\textbf{WN18RR:} WN18RR~\citep{wn18rr} is a knowledge graph, comprising 93,003 triples connecting 40,943 entities through 11 distinct relations, sourced from WordNet. The text data can be obtained from \url{https://github.com/villmow/datasets_knowledge_embedding/tree/master}.

\textbf{FB15k237:} FB15k237~\citep{fb15k237} is also a knowledge graph, containing 310,116 triples involving 14,541 entities and 237 relation types, derived from Freebase. The text data can be obtained from 
\url{https://github.com/villmow/datasets_knowledge_embedding/tree/master}.

\textbf{Molecule datasets:} The molecule datasets include HIV, PCBA, and Tox21, which are from MoleculeNet~\citep{moleculenet}; ChEMBL~\citep{chembl}; and Peptides-func and Peptides-struct, which are from the Long Range Graph Benchmark~\citep{Dwivedi2022}. Each molecule is a graph, where nodes are atoms and edges are chemical bonds. All datasets support multiple tasks, such as predicting biological assays or activities.

\section{Experiment details and hyperparameters}\label{appendix:exp.details}

\textbf{Random walks.} We follow Node2Vec~\citep{Grover2016} and sample biased random walks, which offer greater flexibility in capturing neighborhood information. We set the random walk parameters to $p=1.0$ and $q=0.1$, favoring depth over breadth. Without specification otherwise, for each node, we sample $k=8$ independent walks, each of which has a length $\ell=4$.

\textbf{Transformer.} The Transformer architecture is standard. We use $12$ blocks of multi-head attention plus a subsequent 2-layer MLP, with $12$ heads per block. The hidden dimension is $768$. The projectors for node/edge/dataset features are all linear layers.

\textbf{Pre-training.} We employ AdamW~\citep{adamw} as the optimizer with a learning rate \texttt{5e-5} and weight decay \texttt{1e-2}. We also use a linear learning-rate scheduler with $100$ warm-up steps and a scaling factor of $0.1$ subsequently. Additionally, we apply gradient clipping with a maximum norm of $1.0$ and gradient accumulation of $8$ steps.

\textbf{Dataset mixture.} We use a maximum of 10 datasets to pre-train \model. The multiplier $\alpha_D$ for each dataset $D$ is: PubMed (3.0), Products (0.5), WikiCS (2.0), Arxiv (0.7), WN18RR (0.8), FB15k237 (0.1), PCBA (0.2), ChEMBL (0.1), HIV (1.0), and Tox21 (2.0).

\textbf{(Downstream) Cross-domain cross-task learning.} The task heads are MLPs detailed in Section~\ref{appendix:task_head}, with hyperparameteres listed in Table~\ref{tab:downstream_head}. The optimizer is AdamW with a learning rate \texttt{1e-3} and all other hyperparameters follow PyTorch's default values.

\begin{table}[t]
  \caption{Hyperparameters for downstream tasks. All$^*$ indicates all data in the current split is loaded.}
  \label{tab:downstream_head}
  \vskip 5pt
  \centering
  \small
  \sc
  \begin{tabular}{lcccccc}
    \toprule
    Dataset & \#Layers & Dim. & Dropout & Batch size & Loss & Metric \\
    \midrule
    Cora & 2 & 256 & 0.5 & All$^*$ & NLL & Acc. \\
    citeseer & 2 & 256 & 0.5 & All$^*$ & NLL & Acc. \\
    PubMed & 4 & 768 & 0.5 & 2,048 & NLL & Acc. \\
    Arxiv & 4 & 768 & 0.5 & 2,048 & NLL & Acc. \\
    WikiCS & 2 & 768 & 0.5 & 2,048 & NLL & Acc. \\
    Products & \multicolumn{6}{c}{not used in downstream} \\
    WN18RR & 3 & 256 & 0.15 & 2,048 & NLL & Acc. \\
    FB15k237 & 3 & 256 & 0.5 & 2,048 & NLL & Acc. \\
    HIV & 3 & 768 & 0.5 & 2,048 & BCE & AUC \\
    PCBA & 3 & 768 & 0.15 & 2,048 & BCE & AUC \\
    ChEMBL & \multicolumn{6}{c}{not used in downstream} \\
    Tox21 & 3 & 768 & 0.5 & 2,048 & BCE & AUC \\
    Peptides-func & 4 & 768 & 0.15 & 1,024 & BCE & AUC \\
    Peptides-struct & 4 & 768 & 0.15 & 1,024 & MSE & MAE \\
    \bottomrule
  \end{tabular}
\end{table}

\textbf{(Downstream) Transfer learning.} The hyperparameters are identical to those used in cross-domain cross-task learning. Note that to demonstrate transferability, we vary the pre-training datasets. Hence, the multipliers are set to zero except for those participating in the pre-training. This ensures that no information from non-target datasets is accessed during pre-training.

\textbf{(Downstream) Few-shot leanring.} It consists of two stages: encoder--decoder training and encoder--scorer inference (see Section~\ref{appendix:task_head} for more details). The encoder--decoder training phase is similar to cross-domain cross-task learning, with the exception that the number of training examples is very limited. Specifically, for Cora and WN18RR: 1 example per class; Arxiv: 5 examples per class; and FB15K237: 30 examples per class.

\textbf{Evaluation metric.} The molecule datasets except Peptides-struct use AUC while Peptides-struct uses MAE; and other datasets use accuracy. See Table~\ref{tab:downstream_head}.

\section{Prompt templates for feature extraction}\label{appendix:feature}
We use Llama2-7b to extract input features for \model from node/edge/dataset descriptions. These descriptions come from two primary sources: (1) websites cited in the study of text-attributed graphs~\citep{he2023harnessing,Chen2023a} and (2) hand-crafted text outlined in OFA~\citep{Liu2024}. Details on the sources of text attributes are provided in Section \ref{appendix:dataset}. The text prompts for feature extraction in each dataset are summarized in Table~\ref{tab:prompt}.

\begin{table}[h]
  \caption{Prompting text for feature extraction.}
  \label{tab:prompt}
  \vskip 5pt
  \centering
  \small
  \sc
  \resizebox{\linewidth}{!}{
  \begin{tabular}{ll|l}
    \toprule
    \multirow{3}{*}{Cora} & Node & ``paper title and abstract: '' + \{node desc\} \\
    & Edge & N/A \\
    & Dataset & ``node classification on paper's category in computer science.'' \\
    \midrule
    \multirow{4}{*}{citeseer} & Node & ``paper title and abstract: '' + \{node desc\} \\
    & Edge & N/A \\
    & \multirow{2}{*}{Dataset} & ``node classification on paper's category from citeSeer dataset; '' \\
    & & + ``a citation network of scientific papers.'' \\
    \midrule
    \multirow{3}{*}{PubMed} & Node & ``paper title and abstract: '' + \{node desc\} \\
    & Edge & N/A \\
    & Dataset & ``node classification on paper's category in biomedical domain.'' \\
    \midrule
    \multirow{3}{*}{Arxiv} & Node & ``paper title and abstract: '' + \{node name\} + `` : '' + \{node desc\} \\
    & Edge & N/A \\
    & Dataset & ``node classification of literature category collected from Arxiv platform.'' \\
    \midrule
    \multirow{3}{*}{WikiCS} & Node & ``wikipedia entry name: '' + \{node name\} + ``. entry content: '' + \{node desc\} \\
    & Edge & N/A \\
    & Dataset & ``node classification of wikipedia entry category.'' \\
    \midrule
    \multirow{4}{*}{Products} & Node & ``product: '' + \{node name\} + ``. description: '' + \{node desc\} \\
    & Edge & N/A \\
    & \multirow{2}{*}{Dataset} & ``node classification on product's category collected '' \\
    & & + ``from an Amazon product co-purchasing network.'' \\
    \midrule
    \multirow{4}{*}{WN18RR} & Node & ``entity and entity description: '' + \{node desc\} \\
    & \multirow{2}{*}{Edge} & ``relation between two entities: '' + \{edge desc\} \\ 
    & & ``relation between two entities. The inverse relation of '' + \{edge desc\} \\
    & Dataset & ``relation type prediction between the connected entities of a subset of WordNet.'' \\
    \midrule
    \multirow{5}{*}{FB15k237} & \multirow{2}{*}{Node} & ``entity name: ''
    + \{node name\}
    + ``, entity alternatives: ''
    + \{node name 2\} \\
    & & + ``. entity descriptions: ''
    + \{node desc\} \\
    & \multirow{2}{*}{Edge} & ``relation between two entities: '' + \{edge desc\} \\ 
    & & ``relation between two entities. The inverse relation of '' + \{edge desc\} \\
    & Dataset & ``relation type prediction between the connected entities of Freebase entity pairs.'' \\
    \midrule
    \multirow{3}{*}{HIV} & Node & \{node/atom name \& properties\} \\
    & Edge & \{edge/chemical bond name \& properties\} \\
    & Dataset & ``graph classification on molecule property on dataset: HIV.'' \\
    \midrule
    \multirow{3}{*}{PCBA} & Node & \{node/atom name \& properties\} \\
    & Edge & \{edge/chemical bond name \& properties\} \\
    & Dataset & ``graph classification on molecule property on dataset: PCBA.'' \\
    \midrule
    \multirow{3}{*}{ChEMBL} & Node & \{node/atom name \& properties\} \\
    & Edge & \{edge/chemical bond name \& properties\} \\
    & Dataset & ``graph classification on molecule property on dataset: CHEMBL.'' \\
    \midrule
    \multirow{3}{*}{Tox21} & Node & \{node/atom name \& properties\} \\
    & Edge & \{edge/chemical bond name \& properties\} \\
    & Dataset & ``graph classification on molecule property on dataset: TOX21.'' \\
    \midrule
    \multirow{3}{*}{Peptides-func} & Node & \{node/atom name \& properties\} \\
    & Edge & \{edge/chemical bond name \& properties\} \\
    & Dataset & ``graph classification on molecule property on dataset: Peptides-func.'' \\
    \midrule
    \multirow{3}{*}{Peptides-struct} & Node & \{node/atom name \& properties\} \\
    & Edge & \{edge/chemical bond name \& properties\} \\
    & Dataset & ``graph classification on molecule property on dataset: Peptides-struct.'' \\
    \bottomrule
  \end{tabular}}
\end{table}

\section{More results in ablation study}\label{appendix:ablation}

\subsection{Random walk vs neighborhood sampling}\label{app:ablation.walk.vs.neighbor}
We compare two sequence inputs in Section~\ref{sec:exp.ablation}. The results are reported in Table~\ref{tab:sampler}.

\begin{table}[t]
  \caption{
  Performance comparison between random walk sampler and neighbor sampler. For Peptides-struct, the lower the better; while for others, the higher the better.
  }
  \label{tab:sampler}
  \vskip 5pt
  \centering
  \small
  \sc
  \setlength{\tabcolsep}{4pt}
  \resizebox{\linewidth}{!}{
  \begin{tabular}{ccc|ccc|cc|cccc}
    \toprule
        & $k$ & $\ell$ & PubMed & WikiCS & Arxiv & WN18RR & FB15k237 & HIV & Tox21 & Pep-F & Pep-S \\
    \midrule
        \multirow{5}{*}{Walk} 
        & 16 & 2 & \textbf{76.75} & \textbf{80.80} & \ul{75.15} & 94.32 & \textbf{95.06} & \ul{75.24} & \textbf{74.38} & 88.31 & 0.2599 \\
        & 4 & 4 & \ul{76.32} & 79.51 & 74.94 & 94.95 & 94.78 & \textbf{75.78} & 73.23 & 86.91 & 0.2703 \\
        & 8 & 4 & 73.03 & 80.45 & \textbf{75.16} & \ul{95.20} & 94.95 & 74.03 & \ul{74.16} & 88.21 & 0.2582 \\
        & 4 & 8 & 75.20 & 79.74 & 74.71 & 94.50 & 94.87 & 70.70 & 73.93 & \ul{88.57} & \bf{0.2561} \\
        & 8 & 8 & 75.79 & 80.25 & 74.87 & 94.78 & 94.86 & 75.14 & 73.81 & \textbf{90.48} & \ul{0.2574} \\
    \midrule
        \multirow{2}{*}{Neighbor} & \multicolumn{2}{c|}{$[4,4,2]$} & 74.23 & \ul{80.46} & \ul{75.15} & \textbf{95.55} & \ul{95.05} & 71.78 & 72.40 & 87.58 & 0.2617 \\
         & \multicolumn{2}{c|}{$[8,4]$} & 71.40 & 80.02 & 75.05 & 94.95 & 95.02 & 72.81 & 73.13 & 87.37 & 0.2625 \\
    \bottomrule
    \end{tabular}}
\end{table}

\subsection{Contextual loss}\label{app:ablation.loss}
We evaluate the proposed contextual loss~\eqref{eqn:context.loss} against a few alternatives, including other forms of context prediction and the possible inclusion of reconstruction losses. For ease of comparison, we repeat~\eqref{eqn:context.loss} here:
\begin{equation*}
  \tL_{sample} = -\frac{1}{\ell} \sum_{j=1}^{\ell} \Big( \log\text{MLP}( \vh_0\odot\vh_{ctx}^{(j)} )
  + \sum_{\forall other} \log(1-\text{MLP}( \vh_0\odot\vh_{other}^{(j)} )) \Big),
\end{equation*}
recalling that $\vh_0$ is the node representation computed by \model, $\vh_{ctx}^{(j)}$ is its $j$-hop context, and $\vh_{other}^{(j)}$ is the $j$-hop context of other nodes in the training batch. Experiment results are summarized in Tables~\ref{tab:ctx_loss_1} and~\ref{tab:mask_loss}; they are discussed in the main text (Section~\ref{sec:exp.ablation}).

\label{appendix:ctx_loss}
Several variants of the contrastive loss that performs context prediction exist in the graph literature. We consider DGI~\citep{Velickovic2019}, GraphPrompt~\citep{graphprompt}, and MaskGAE~\citep{maskgae} and adapt their losses for sequence models. For clarity, we use $\vh_{ctx}^{(j),+}$ to denote the context for the concerned node; $\vh_{ctx}^{(j),-}$ to denote the negative context; and when positive and negative are not distinguished, we use the notation $\vh_{ctx}^{(j),i}$. In the latter notation, $i$ ranges over all nodes in the batch.

\textbf{DGI.} In this design, each node is associated with two sequences: one generated by the random walk sampler and the other by random nodes. The first sequence computes a positive context $\vh_{ctx}^{(j),+}$ and is assigned label 1; while the other computes a negative context $\vh_{ctx}^{(j),-}$ and is assigned label 0. The task is to predict the correct context and hence the loss is the binary cross-entropy:
\begin{equation}
  \tL_{sample} = -\frac{1}{2\ell} \sum_{j=1}^{\ell} \Big( \log g( \vh_0, \vh_{ctx}^{(j),+} )
  + \log(1-g( \vh_0, \vh_{ctx}^{(j),-} )) \Big),
\end{equation}
where $g$ is the Sigmoid function applied to the inner product of two representations.

\textbf{GraphPrompt.} In this design, the loss is the log-softmax of the similarity between the representations of the node and all the contexts in the training batch (one positive and all others negative):
\begin{equation}
  \tL_{sample} = -\frac{1}{\ell} \sum_{j=1}^{\ell}  \log \frac{ \exp (\text{sim}( \vh_0, \vh_{ctx}^{(j),i} )/\tau) }{ \sum_{i} \exp (\text{sim}( \vh_0, \vh_{ctx}^{(j),i} )/\tau) },
\end{equation}
where \texttt{sim} denotes cosine similarity.

\textbf{MaskGAE.} The original MaskGAE loss contrasts positive and negative edges. We adapt this idea by predicting the $j$-hop context for the $(j-1)$-hop context, rather than for the root node:
\begin{equation}
  \tL_{sample} = -\frac{1}{\ell} \sum_{j=1}^{\ell}  \Big( \log g( \vh_{ctx}^{(j-1)}, \vh_{ctx}^{(j),+} )
  + \sum_{\forall\,\vh_{ctx}^{(j),-}} \log(1-g( \vh_{ctx}^{(j-1)}, \vh_{ctx}^{(j),-} )) \Big),
\end{equation}
where the zero-hop context is simply the root node itself: $\vh_{ctx}^{(0)} := \vh_0$. Similar to DGI, $g$ is the Sigmoid function applied to the inner product of two representations.

In Table~\ref{tab:ctx_loss_1}, we observe that the proposed contextual loss performs better than the three alternatives.

\begin{table}[t]
  \caption{Performance comparison between contextual losses.}
  \label{tab:ctx_loss_1}
  \vskip 5pt
  \centering
  \small
  \sc
  \begin{tabular}{l|ccccc}
    \toprule
        & Cora & citeseer & PubMed & WikiCS & Arxiv \\
    \midrule
        Our loss~\eqref{eqn:context.loss} & \textbf{78.02 ± 0.37} & \textbf{75.46 ± 0.45} & 76.26 ± 0.53 & 80.27 ± 0.19 & \textbf{75.14 ± 0.15} \\
        DGI & 76.78 ± 0.61 & 74.37 ± 0.88 & \ul{76.99 ± 0.17} & \ul{80.34 ± 0.26} & \ul{75.13 ± 0.12} \\
        GraphPrompt & 74.97 ± 0.35 & 74.45 ± 0.60 & \textbf{77.86 ± 0.19} & 80.01 ± 0.29 & 75.08 ± 0.20 \\
        MaskGAE & \ul{77.51 ± 0.28} & \ul{74.87 ± 0.43} & 74.81 ± 0.30 & \textbf{80.48 ± 0.13} & 75.05 ± 0.10 \\
    \bottomrule
    \end{tabular}
    \vskip 8pt
  \begin{tabular}{l|cccc}
    \toprule
        & WN18RR & FB15k237 & HIV & Tox21 \\
    \midrule
        Our loss~\eqref{eqn:context.loss} & \textbf{95.25 ± 0.00} & 94.47 ± 0.06 & \textbf{75.15 ± 2.39} & \ul{74.49 ± 1.40} \\
        DGI & 94.55 ± 0.32 & 94.71 ± 0.11 & 72.74 ± 1.30 & 73.30 ± 0.43 \\
        GraphPrompt & 94.28 ± 0.07 & \textbf{94.89 ± 0.08} & 73.46 ± 2.40 & \textbf{75.55 ± 0.54} \\
        MaskGAE & \ul{94.64 ± 0.13} & \ul{94.79 ± 0.10} & \ul{74.62 ± 1.77} & 72.83 ± 1.07 \\
    \bottomrule
    \end{tabular}
\end{table}

\subsection{Reconstruction loss}
\label{appendix:mask_loss}

\begin{table}[t]
    \caption{Performance comparison between training losses.}
      \label{tab:mask_loss}
      \vskip 5pt
      \centering
      \small
      \sc
    \resizebox{\linewidth}{!}{
    \begin{tabular}{l|ccccc}
    \toprule
        & Cora & citeseer & PubMed & WikiCS & Arxiv \\
    \midrule
        Context prediction~\eqref{eqn:context.loss} & \ul{78.02 ± 0.37} & \textbf{75.46 ± 0.45} & \textbf{76.26 ± 0.53} & \ul{80.27 ± 0.19} & \ul{75.14 ± 0.15} \\
        + Token reconstruction & \textbf{78.18 ± 0.52} & \ul{75.25 ± 0.71} & \ul{75.65 ± 0.18} & \textbf{80.74 ± 0.34} & \textbf{75.33 ± 0.19} \\
        + Position reconstruction & 76.61 ± 0.44 & 74.69 ± 0.43 & 74.59 ± 0.20 & 80.25 ± 0.23 & 75.00 ± 0.13 \\
    \bottomrule
    \end{tabular}}
    \vskip 8pt
    \resizebox{0.9\linewidth}{!}{
    \begin{tabular}{l|cccc}
    \toprule
        & WN18RR & FB15k237 & HIV & Tox21 \\
    \midrule
        Context prediction~\eqref{eqn:context.loss} & \textbf{95.25 ± 0.00} & 94.47 ± 0.06 & \ul{75.15 ± 2.39} & \textbf{74.49 ± 1.40} \\
        + Token reconstruction & \ul{94.54 ± 0.32} & \textbf{94.85 ± 0.07} & 72.77 ± 1.07 & 74.41 ± 0.48 \\
        + Position reconstruction & 94.38 ± 0.25 & \ul{94.67 ± 0.16} & \textbf{75.33 ± 0.91} & \ul{74.42 ± 0.76} \\
    \bottomrule
    \end{tabular}}
\end{table}

In addition to context prediction, we consider reconstruction, which is common in the pre-training of LLM encoders, such as the mask-token prediction in BERT~\citep{bert}. In this setup, one randomly masks 15\% of the nodes in the input sequence and replaces them with a special token. Then, the reconstruction is to have \model predict these tokens in the output.

We compare three configurations: (1) no reconstruction loss (i.e., context prediction only); (2) adding the reconstruction loss for token embeddings of randomly masked nodes; and (3) adding the reconstruction loss for positional embeddings of randomly masked nodes. In Table~\ref{tab:mask_loss}, we observe that adding a reconstruction loss causes marginal differences. Hence, for cost reasons, we do not use a reconstruction loss.

\subsection{Architecture design}
\label{appendix:model_design}

Our architecture design follows the standard decoder-only Transformer (GPT-2 \citep{gpt2}), with two modifications: (1) using edge features in each Transformer layer (see Section~\ref{sec:seq_enc}); and (2) a customized attention mask (see Section~\ref{sec:attn_mask}). This ablation study evaluates the impact of the changes, as summarized in Table~\ref{tab:model_design}.

\textbf{Edge feature input.} 
In our model, edge features are incorporated at each transformer layer, similar to GNNs \citep{Hu2020a}, by attaching a projection head to every transformer layer.
Alternatively, as in prior work, edge features can be introduced at the input level by adding tokenized edge embeddings alongside node and positional embeddings.

Using edge features only at the input level
yields inconsistent effects on node- and link-level tasks--two datasets show improvements and two show degradation--due to the absence of edge features in pre-training on these datasets. 
Worse, it consistently harms the performance of graph-level tasks.

\textbf{Full attention mask.} 
Replacing our attention mask (Figure~\ref{fig:mask}) with the standard full attention mask used in BERT~\citep{bert} allows each contextual node to attend to any other node sampled by other random walks.
However, doing so will introduce spurious connections;
for instance, nodes from two random walks moving in opposite directions--typically weakly connected--would be treated as fully connected under the full attention mask.

As a result, using full attention results in performance degradation on five out of six benchmarks. Nevertheless, the impact is moderate, as many of the introduced connections are either correct or only mildly disruptive due to all nodes residing within the same local neighborhood.

\begin{table}[H]
  \caption{Performance comparison on different model architectures.}
  \label{tab:model_design}
  \vskip 5pt
  \centering
  \small
  \sc
  \resizebox{\linewidth}{!}{
  \begin{tabular}{l|cccccc}
    \toprule
        & WikiCS & Arxiv & WN18RR & FB15k237 & HIV & Tox21 \\
    \midrule
        Our model & \ul{80.27 ± 0.19} & \textbf{75.14 ± 0.15} & \textbf{95.25 ± 0.00} & 94.47 ± 0.06 & \textbf{75.15 ± 2.39} & \textbf{74.49 ± 1.40} \\
        Edge feat. input & \textbf{80.93 ± 0.26} & \ul{74.92 ± 0.20} & \ul{94.67 ± 0.19} & \textbf{94.81 ± 0.01} & 73.15 ± 2.95 & 73.54 ± 0.27 \\
        Full attn. mask & 79.55 ± 0.12 & 74.62 ± 0.21 & 94.44 ± 0.14 & \ul{94.78 ± 0.11} & \ul{73.42 ± 3.43} & \ul{73.78 ± 1.08} \\
    \bottomrule
    \end{tabular}}
\end{table}

\end{document}